\documentclass[lettersize,journal]{IEEEtran}
\usepackage{amsmath,amsfonts}
\usepackage{array}
\usepackage[caption=false,font=normalsize,labelfont=sf,textfont=sf]{subfig}

\usepackage{algorithm}
\usepackage{algorithmicx}
\usepackage{algpseudocode}
\usepackage{amsmath}
\DeclareMathOperator*{\argmax}{arg\,max}

\usepackage{textcomp}
\usepackage{stfloats}
\usepackage{url}
\usepackage{verbatim}
\usepackage{graphicx}
\usepackage{hyperref}
\usepackage{algorithm}
\usepackage{cuted}   
\newcommand\numeq[1]%
{\stackrel{\scriptscriptstyle(\mkern-1.5mu#1\mkern-1.5mu)}{=}}
\newcommand\numeqs[1]%
{\stackrel{\scriptscriptstyle(\mkern-1.5mu#1\mkern-1.5mu)}{\leq}}
\usepackage{multirow}
\usepackage{array}
\usepackage{booktabs}
\usepackage{amsmath}  
\usepackage{amsthm}   

\usepackage{titlesec}
\usepackage{cite}

\usepackage[dvipsnames]{xcolor}
\def\BibTeX{{\rm B\kern-.05em{\sc i\kern-.025em b}\kern-.08em
    T\kern-.1667em\lower.7ex\hbox{E}\kern-.125emX}}
\usepackage{balance}
\newtheorem{theorem}{Theorem}
\newtheorem{remark}{Remark}
\newtheorem{note}{Note}
\newtheorem{lemma}{Lemma}

\newcommand{\E}{\mathbb{E}}

\usepackage{tikz}

\usetikzlibrary{shadows}

\begin{document}
\title{Double Successive Over-Relaxation Q-Learning with an Extension to Deep Reinforcement Learning}
\author{Shreyas S. R.$^{\dagger}$
	\thanks{$^{\dagger}$Research Scholar, Department of Mathematics, Indian Institute of Technology (IIT) Indore, Madhya Pradesh 452020, India. Email: \texttt{shreyassr123@gmail.com}}}

\maketitle

\begin{abstract}
Q-learning is a widely used algorithm in reinforcement learning (RL), but its convergence can be slow, especially when the discount factor is close to one. Successive Over-Relaxation (SOR) Q-learning, which introduces a relaxation factor to speed up convergence, addresses this issue but has two major limitations: In the tabular setting, the relaxation parameter depends on transition probability, making it not entirely model-free, and it suffers from overestimation bias. To overcome these limitations, we propose a sample-based, model-free double SOR Q-learning algorithm. Theoretically and empirically, this algorithm is shown to be less biased than SOR Q-learning. Further, in the tabular setting, the convergence analysis under boundedness assumptions on iterates is discussed. The proposed algorithm is extended to large-scale problems using deep RL. Finally, both the tabular version of the proposed algorithm and its deep RL extension is tested on benchmark examples. 
\end{abstract}


\section{Introduction}

Reinforcement learning (RL) is a commonly employed approach for decision-making in a variety of domains, wherein agents learn from their interactions with the environment. The primary goal of RL is to find an optimal policy, which involves solving Bellman's optimality equation. Depending on the amount of information available, various algorithms are proposed in both dynamic programming and RL to solve Bellman's optimality equation. When the complete information of the Markov Decision Process (MDP) is available (model-based), techniques from dynamic programming, such as value iteration and policy iteration, are utilized to find the optimal value function, thereby, an optimal policy \cite{MR1270015}. In cases where the MDP is not fully known (model-free), RL techniques are utilized \cite{MR3889951}. Q-learning (QL), introduced by Watkins in 1992, is a fundamental algorithm in RL used to find an optimal policy in model-free environments \cite{watkins1992q}. Despite its wide application in several areas \cite{xie,liu, wu,9720124}, Q-learning suffers from slow convergence \cite{azar2011speedy, kamanchi2019successive} and over-estimation bias \cite{thrun2014issues, hasselt2010double, 8695133, wu2020reducing}.

In 1973, D. Reetz proposed a successive over-relaxation (SOR) technique to accelerate the convergence of value iteration in dynamic programming \cite{MR398532}. This technique involves constructing a modified Bellman optimality equation using an SOR parameter. The SOR parameter in \cite{MR398532} depends on the transition probability of the self-loop of an MDP. Furthermore, it was demonstrated in \cite{MR398532} that, with a proper choice of the SOR parameter, the convergence of value iteration can be significantly improved.

Recently, this SOR technique has been incorporated into Q-learning, resulting in a new algorithm called SOR Q-learning (SORQL) \cite{kamanchi2019successive}. In essence, SOR Q-learning is a variant of the standard Q-learning algorithm designed to accelerate the convergence of Q-learning. The motivation for the SORQL algorithm stems from the fact that the underlying deterministic mapping of SORQL can be shown to have a contraction factor less than that of the standard Q-Bellman operator \cite{kamanchi2019successive}, thereby speeding up the convergence of the Q-learning algorithm. However, the SOR parameter in SORQL depends on the transition probability of a self-loop in the MDP, making the algorithm not entirely model-free. Furthermore, we have identified both theoretically and empirically that SOR Q-learning suffers from overestimation bias. To address these issues, in this manuscript, we propose model-free double SOR Q-learning (MF-DSORQL), which combines the idea of double Q-learning \cite{hasselt2010double} with the SOR technique. The advantages of the proposed algorithm are that it is entirely model-free and has been shown, both theoretically and empirically, to have less bias than SORQL. In addition, this work provides the convergence analysis of the proposed MF-DSORQL algorithm using stochastic approximation (SA) techniques.

 Motivated from the work in \cite{mnih2013playing}, and \cite{9206598}, wherein the double Q-learning and SOR Q-learning in the tabular version extended to large-scale problems using the function approximation. In this manuscript, we extend the tabular version of the proposed algorithm and present a double successive over-relaxation deep Q-network (DSORDQN). The proposed algorithm is tested on roulette \cite{hasselt2010double}, multi-armed bandit problem \cite{8695133}\cite{tan2024q}, OpenAI gym's CartPole \cite{brockman2016openai}, LunarLander \cite{brockman2016openai} environments, Atari games and a maximization example where the state space is in the order of $10^9$ \cite{weng2020mean}. The experiments corroborate the theoretical findings.
 \vspace*{-0.45cm}
 \subsection{Related Works and Brief Literature Review}
 Several variants of Q-learning algorithms are available in the literature. To mention a few, in \cite{azar2011speedy}, speedy Q-learning (SQL) algorithm is developed to fasten the convergence of the Q-learning algorithm. At every iteration, SQL uses two successive Q-estimates in its update rule. Recently, the concept of SOR has been applied to the SQL algorithm, leading to the development of a new algorithm called generalized speedy Q-learning \cite{john2020generalized}. The DQL proposed in \cite{hasselt2010double} to address the over-estimation of Q-learning was later improved in \cite{zhang2017weighted} using a weighted double Q-learning approach. Specifically, \cite{zhang2017weighted} uses a weighted combination of estimates of DQL \cite{zhang2017weighted}. In \cite{8695133},  bias-corrected Q-learning algorithm is proposed wherein a bias correction term is constructed and subtracted from the Q-learning iteration. Also, self-corrected Q-learning is proposed in \cite{zhu2021self}, where a self-correcting estimator is defined to control the bias. Recently, a generalization of the Q-learning algorithm, known as maxmin Q-learning, has been proposed to address the maximisation bias, which uses $N$ estimates with $N = 1$ corresponding to the Q-learning \cite{Lan}. Unlike most of the algorithms mentioned above, the proposed method solves a modified Bellman’s equation. In addition, contrast to the existing works, which address the over-estimation issue of Q-learning or slower convergence of Q-learning separately, the proposed method aims to tackle both issues simultaneously. 

Lately, in \cite{diddigi2022generalized}, the SORQL algorithm has been suitably modified to handle two-player zero-sum games, and consequently, a generalized minimax Q-learning algorithm was proposed. The convergence of several RL algorithms depends on the theory developed in SA \cite{tsitsiklis1994asynchronous,6796861, MR2442439,singh2000convergence, diddigi2022generalized}. Many of the iterative schemes in the literature assume the boundedness and proceed to show the convergence. For instance, in \cite{diddigi2022generalized}, the convergence of the generalized minimax Q-learning algorithm was shown to converge under the boundedness assumption. In this manuscript, we take the same approach to show the convergence of the proposed algorithms in the tabular setting.

The key contributions are summarized as follows
\begin{itemize}
	\item A model-based SOR Q-learning algorithm in \cite{kamanchi2019successive} is modified into a model-free version (MF-SORQL) to make it more practically feasible.
	\item The issue of overestimation in the model-based SOR Q-learning algorithm is identified and demonstrated both theoretically and empirically. Based on these findings, a new model-free algorithm, known as MF-DSORQL, is proposed.
	\item The convergence of both MF-SORQL and MF-DSORQL algorithms is established using techniques from stochastic approximation theory.
	\item The proposed  MF-DSORQL algorithm is extended to deep RL setting and evaluated, along with its tabular version, on benchmark examples.
\end{itemize}

The paper is organized as follows: Section II covers preliminaries and convergence prerequisites. Section III presents the proposed algorithms. Section IV provides convergence analysis. Section V demonstrates numerical experiments, and Section VI concludes the paper.
\vspace{-0.2cm}
\section{Preliminaries}

MDP provides a mathematical representation of the underlying control problem and is defined as follows: Let $S$ be the set of finite states and $A$ be the finite set of actions. The transition probability for reaching the state $j$ when the agent takes action $a$ in state $i$ is denoted by $p^j_{ia}$. Similarly, let $r^j_{ia}$ be a real number denoting the reward the agent receives when an action $a$ is chosen at state $i$, and the agent transitions to state $j$. Finally, the discount factor is $\gamma$, where $0\leq \gamma <1$. A policy $\pi$ is a mapping from $S$ to $\Delta(A)$, where $\Delta(A)$ denotes the set of all probability distributions over $A$. The solution of the Markov decision problem is an optimal policy $\pi^*$. The problem of finding an optimal policy $\pi^*$ reduces to finding the fixed point of the following operator $U: \mathbb{R}^{|S\times A|} \rightarrow \mathbb{R}^{|S\times A|}$ defined as
\begin{equation}
UQ(i,a) = r(i,a)+\gamma\sum_{j=1}^{|S|}p^j_{ia} \mathcal{M}Q(j)
\end{equation}
where
\begin{equation}
r(i,a) = \sum_{j=1}^{|S|}p^j_{ia}r^j_{ia} \quad \text{and} \quad \mathcal{M}Q(j) = \max_{b \in A} Q(j,b).
\end{equation}

To simplify the equations, we defined the max-operator, $\max_{a \in A} Q(i, a)$ as $\mathcal{M}Q(i)$ for all $i \in S$. Also, $R_{\max} = \max_{(s,a,s')}|r^{s'}_{sa}|, \forall (s,a,s') \in S \times A \times S$. These notations will be used throughout the manuscript as necessary. Note that $U$ is a contraction operator with $\gamma$ as the contraction factor, and $Q^*$ serves as the optimal action-value function of the operator $U$. Moreover, one can obtain an optimal policy $\pi^*$ from the $Q^*$, using the  relation $\pi^*(i) \in \text{arg}\, \mathcal{M} Q^*(i), \;\forall i \in S.$ 

Q-learning algorithm is a SA version of Bellman's operator $U$. Given a sample $\{i,a,j,r^j_{ia}\}$, current estimate $Q_n(i,a)$, and a suitable step-size $\beta_n(i,a)$, the update rule of Q-learning is as follows
\begin{equation}
Q_{n+1}(i,a) = Q_n(i,a)+\beta_n(i,a)   (r^j_{ia}+\gamma \, \mathcal{M} Q_n(j)-Q_n(i,a)).
\end{equation}

Under suitable assumptions, the above iterative scheme is shown to converge to the fixed point of $U$ with probability one (w.p.1) \cite{watkins1992q,tsitsiklis1994asynchronous,6796861}. Recently, in \cite{kamanchi2019successive}, a new Q-learning algorithm known as successive over-relaxation Q-learning (SORQL) is proposed to fasten the convergence of Q-learning. More specifically, a modified Bellman's operator is obtained using the successive over-relaxation technique. In other words, instead of finding the fixed point of $U$, the fixed point of the modified Bellman's operator $U_w:\mathbb{R}^{|S\times A|}\rightarrow \mathbb{R}^{|S\times A|}$ is evaluated and is defined as  
\begin{align}
&(U_wQ)(i,a) \nonumber\\&=w\bigg(r(i,a)+\gamma \sum_{j=1}^{|S|}p^j_{ia}\mathcal{M}Q(j)\bigg)+(1-w)\mathcal{M}Q(i)
\end{align}
where $0<w\leq w^*$, and $w^*=\min\limits_{i,a}\left\{\frac{1}{1-\gamma p^i_{ia}}\right\}.$
The contraction factor of the map $U_w$ is $1-w+w\gamma$. It is interesting to note that for a suitable SOR parameter $w$, one can show that $1-w+w\gamma < \gamma$ (Lemma 4 in \cite{kamanchi2019successive}). Consequently, the fixed-point iterative scheme corresponding to $U_w$ converges faster compared to the operator $U$. The SORQL algorithm is built using the operator $U_w$ and can be viewed as the stochastic approximation version of the fixed-point iterative scheme $Q_{n+1}=U_w(Q_n)$. Therefore, one can expect that the SORQL algorithm will converge faster than that of the Q-learning. Although the fixed points of $U$ and $U_w$ may not be the same, the following relation between their optimal Q-values ensures that they yield the same optimal policies
\begin{equation}
\mathcal{M} Q^*_w(i)=\mathcal{M} Q^*(i), \quad \forall i \in S
\end{equation}
where $Q^*$ and $Q^*_w$ are the fixed points of $U$ and $U_w$, respectively. Under suitable assumptions, it was proved that the SORQL algorithm given by the update rule 
\begin{align}
&Q_{n+1}(i,a) \nonumber\\
&= (1-\beta_n(i,a) Q_n(i,a) + \beta_n(i,a) \Big[ w \Big( r^j_{ia} + \gamma \, \mathcal{M} Q_n(j) \Big) \nonumber\\
&\hspace*{4.8cm} + (1-w) \mathcal{M} Q_n(i) \Big]
\end{align}
where $0 < w \leq w^*$, and $0\leq \beta_n(i,a)\leq 1$ converges to the fixed point of $U_w$ w.p.1 \cite{kamanchi2019successive}. 

Note that the successive relaxation factor in the above equation depends on the transition probability, which makes it not entirely model-free. Owing to this and the natural concerns regarding the over-estimation of SORQL. This manuscript proposes a model-free variant of the SORQL algorithm and a model-free double SOR Q-learning algorithm. 

To prove the convergence of the proposed algorithms in the tabular setting, the following result from  \cite{singh2000convergence} will be useful, and we conclude this section by stating it here as lemma.
\begin{lemma}(Lemma 1, in \cite{singh2000convergence})\label{fl}
	Let a random process $ \left(\Phi_{n},J_n,\beta_n\right), $ $n \geq 0$, where $\Phi_{n},J_n,\beta_n:X\rightarrow \mathbb{R}$ satisfy the following relation: 
	$
	\Phi_{n+1}(x)=(1-\beta_n(x))\Phi_{n}(x)+\beta_n(x)J_n(x),\; \text{where} \hspace{0.2cm} x \in X.$ Suppose $\mathcal{G}_n$ is an increasing sequence of $\sigma$-fields with $\beta_0$, $\Phi_0$ are $\mathcal{G}_0$ measurable and $\beta_n$, $\Phi_n$, $J_{n-1}$ are $\mathcal{G}_{n}$ measurable for $n\geq1$.  
	Then $\Phi_{n}$ $\rightarrow$ $0$ w.p.1 as $n$ $\rightarrow$ $\infty$, under the following conditions:	(1) $X$ is a finite set. (2) Step-size satisfy, $0\leq\beta_n(x)\leq 1$, $\sum_{n=1}^{\infty}\beta_n(x)=\infty$, $\sum_{n=1}^{\infty}\beta^2_n(x)<\infty$ $w.p.1.$
	(3) $\|\E[J_n|\mathcal{G}_n]\| \leq K \|\Phi_n\|+\delta_n$, where $K \in[0,1)$ and $\delta_n$ $\rightarrow$ $0$ $w.p.1$. (4) $Var[ J_n(x)|\mathcal{G}_n]\leq C(1+ \|\Phi_n\|)^2$, where $C$ is some constant. 
\end{lemma}
\section{Proposed Algorithm}
In this section, we first discuss the tabular version of the proposed algorithms, followed by a discussion of the deep RL version.
\subsection{Tabular Version}
This subsection presents the model-free variants of the algorithm SOR Q-learning and double SOR Q-learning.  As previously mentioned, the successive relaxation factor $w$ is dependent on $p^i_{ia}$, which is generally unknown. To remove this dependency on the transition probability of a self-loop, the following scheme is proposed: 
For any $i,j \in S$ and $a \in A$, let $Y_n[i][j][a]$ denote the number of times the states $i,j$ and action $a $ are visited till $n^{th}$ iteration. Further, we assume $Y_0[i][j][a] = 0, \, \forall i,j,a \in S\times S \times A$, and  define
\begin{equation}
\widetilde{P}^j_{ia} = \dfrac{Y_n[i][j][a]}{n}, \quad \quad  \quad  n \geq 1. 
\end{equation}

Using the strong law of large numbers $\widetilde{P}^j_{ia} \rightarrow p^j_{ia}, \, \forall\, i,j,a$ as $n \rightarrow \infty$, w.p.1. Now for $w_0 \in [1,\frac{1}{1-\gamma}]$, and $\alpha(n)$ satisfying step-size condition \textit{(2)} in Lemma \ref{fl}, we consider the following iterative scheme 
\begin{equation}\label{2}
w_{n+1} = w_n + \alpha(n) \bigg(\dfrac{1}{1-\gamma \min\limits_{i,a} \widetilde{P}^j_{ia}}-w_n\bigg), \quad  n\geq 1.
\end{equation}
We rewrite the above iterative scheme as follows
\begin{equation}
w_{n+1} = w_n + \alpha(n) (f(w_n)+\epsilon_n)\quad  \quad  n\geq 1
\end{equation}
where 
\begin{equation}
f(w_n) = \dfrac{1}{1-\gamma \min\limits_{i,a}p^j_{ia}} - w_n, \;\epsilon_n =  \dfrac{1}{1-\gamma \min\limits_{i,a} \widetilde{P}^j_{ia}}-w^*.
\end{equation}
Since $\widetilde{P}^j_{ia} \rightarrow p^j_{ia}$ w.p.1, this implies $\epsilon_n \rightarrow 0$ as $n \rightarrow \infty$ w.p.1. Note that for any $w_1,w_2 \in \mathbb{R}$
\begin{equation}
|f(w_1)-f(w_2)| \leq |w_1 -w_2|.
\end{equation}

The iterate $\{w_n\}$ in \eqref{2} track the the ODE, $\dot{w}=w^*-w$ [\cite{MR2442439}, Section 2.2]. Let $f_{\infty}(w)=\lim_{r \rightarrow \infty}\frac{f(rw)}{r}$. The function $f_{\infty}(w)$ exist and is equal to $-w$. Further, the origin and $w^*$ is the unique globally asymptotically stable equilibrium for the ODE $\dot{w}=f_{\infty}(w)=-w$ and $\dot{w}=w^*-w$, respectively. 

As a result of the above observations and \cite{MR2442439}, one can obtain the following theorem.
\begin{theorem}\label{thm1}
	The iterative scheme defined in \eqref{2}, satisfy  $\sup_n \, \|w_n\| < \infty$ w.p.1. Further, $w_n$ $\rightarrow$ $w^*$ w.p.1 as $n \rightarrow \infty$.
\end{theorem}
\begin{proof}
    The proof is a consequence of Theorem 7 in Chapter 3, and Theorem 2 - Corollary 4 in Chapter 2 \cite{MR2442439}.
\end{proof}
\begin{note}
	Throughout this manuscript, the sequence $w_n$ is assumed to be updated using \eqref{2}.
\end{note}
\begin{remark}\label{rem1}
	The model-free SORQL (MF-SORQL) algorithm is nothing but Algorithm 1 in \cite{kamanchi2019successive}, with $w$ replaced by $w_n$. 
\end{remark}

\begin{algorithm}[ht]

\caption{Model-Free Double SOR Q-Learning (MF-DSORQL)}
\begin{algorithmic}[1]
    \Require $Q^A_0$, $Q^B_0$; $\gamma$, $\beta_n$, $w_n$, $N$, policy $\mu$ producing every $S\times A$ pairs indefinitely.
    \For{$n = 0, 1, \cdots, N-1$}
        \State Choose $a$ at $i_n$ according to $\mu$
        \State Observe samples $i'_n$ and $r'_n$
        \State Update $Q^A_{n+1}$ or $Q^B_{n+1}$ with equal probability
        \If{$Q^A_{n+1}$ is updated}
            \State \hspace*{-0.3cm}$b^* = \arg\max_{b} Q^A_n(i'_n, b)$, $c^* = \arg\max_{c} Q^A_n(i_n, c)$
            \State\hspace*{-0.6cm} $Q^A_{n+1}(i_n, a_n) = (1 - \beta_n(i_n, a_n)) Q^A_n(i_n, a_n) + \beta_n(i_n, a_n) \big( w_n (r'_n + \gamma Q^B_n(i'_n, b^*)) + (1 - w_n) Q^B_n(i_n, c^*) \big)$
        \Else
            \State \hspace*{-0.4cm} $d^* = \arg\max_{d} Q^A_n(i'_n, d)$, $e^* = \arg\max_{e} Q^A_n(i_n, e)$
            \State \hspace*{-0.6cm} $Q^B_{n+1}(i_n, a_n) = (1 - \beta_n(i_n, a_n)) Q^B_n(i_n, a_n) + \beta_n(i_n, a_n) \big( w_n (r'_n + \gamma Q^A_n(i'_n, d^*)) + (1 - w_n) Q^A_n(i_n, e^*) \big)$
         \quad \EndIf
        \State $i_{n+1} = i'_n$
    \EndFor\\
    \Return $Q_{N-1}$
\end{algorithmic}
\end{algorithm}
\begin{remark}\label{remark2}
If the sequence $\{w_n\}_ {n\geq 0}$ in Algorithm 1 is a constant sequence, with $w$ as the constant and $0<w \leq w^*$, we refer to this algorithm as double SOR Q-learning (DSORQL). Therefore, given the samples $\{i,a,j,r^j_{ia}\}$ the update rule for the DSORQL is as follows \\
With probability $0.5$ update $\widetilde{Q}^A$
\begin{equation}
\small
\begin{aligned}
\widetilde{Q}^A_{n+1}(i,a) &= (1 - \beta_n) \widetilde{Q}^A_n(i,a) \\
&\;+ \beta_n \Big[ w \big( r^j_{ia} + \gamma \widetilde{Q}^B_n(j,b^*) \big) + (1 - w) \big( \widetilde{Q}^B_n(i,c^*) \big) \Big]
\end{aligned}
\end{equation}

\noindent else update $\widetilde{Q}^B$
\begin{equation}
\small
\begin{aligned}
\widetilde{Q}^B_{n+1}(i,a) &= (1 - \beta_n) \widetilde{Q}^B_n(i,a) \\
&\;+ \beta_n \Big[ w \big( r^j_{ia} + \gamma \widetilde{Q}^A_n(j,d^*) \big)+ (1 - w) \big( \widetilde{Q}^A_n(i,e^*) \big) \Big]
\end{aligned}
\end{equation}
where $\beta_n=\beta_n(i,a)$, and $b^*, c^*, d^*$, $e^*$ are as in Alg. 1.

\end{remark}
\begin{note}
	The SORQL algorithm discussed in \cite{kamanchi2019successive} and the DSORQL algorithm mentioned in  Remark \ref{remark2} become model-free when $|S|=1$. Therefore, these algorithms can be implemented directly in single-state MDPs, wherein $w$ will be known and satisfy $0<w\leq w^* = \frac{1}{1-\gamma}$.
\end{note}
\subsection{ Deep RL Version} 
In \cite{van2016deep}, the authors highlight the issue of overestimation in DQN on large-scale deterministic problems. They introduce double DQN (DDQN), a modified version of the DQN algorithm that effectively reduces the over-estimation.

 The proposed double SOR deep Q-network (DSORDQN) is a deep RL version of the tabular DSORQL. DSORDQN utilizes the technique of SOR method with DDQN. Specifically, similar to DDQN we use online Q-network $Q(S,a;\theta)$ for selecting the actions, and the target Q-network $Q(S,a;\theta')$, to evaluate the actions. The gradient descent will be performed on the following loss function:
$L(\theta_i) = \E \big[(y_j - Q (s,a;\theta_i))^2\big],$
where $y_j = w (r(i_j,a_j)+\gamma Q(i_{j+1},\argmax_b Q(i_{j+1},b;\theta_i);\theta'_i)+ (1-w)Q(i_{j},\argmax_e Q(i_{j},e;\theta_i);\theta'_i)$ for sample $(i_j,a_j,r(i_j,a_j), i_{j+1})$ from the replay buffer with $w>1$.

 Interestingly, the advantages of DDQN and SORDQN transfer to the proposed DSORDQN algorithm, resulting in more stable and reliable learning at scale.
 
\section{Theoretical Analysis}
At first, we discuss the convergence of model-free SOR Q-learning (Remark \ref{rem1}) under the following assumption:
\textbf{(A1)}: $\| Q_n\| \leq B < \infty$, $\forall n \geq 0$.
\begin{theorem}\label{thm3}
	Suppose \textbf{(A1)} holds. Given an MDP defined as in Section II, and $w_n$ as in \eqref{2}. Let every state-action pair be sampled indefinitely. Then, for sample $\{i,a,j,r^j_{ia}\}$ the update rule of model-free SOR Q-learning algorithm given by
\begin{equation}\label{key22}
\begin{aligned}
Q_{n+1}(i,a) &= (1 - \beta_n(i,a))Q_n(i,a) \\
&\quad + \beta_n(i,a)\left[w_n \left(r^j_{ia} + \gamma \mathcal{M} Q_n(j)\right) \right. \\
&\quad \left. + (1 - w_n) \left(\mathcal{M} Q_n(i)\right) \right]
\end{aligned}
\end{equation}
	converges w.p.1 to the fixed point of $U_{w^*}$, where  $\sum_{n}\beta_n(i,a)=\infty,\; \sum_{n}\beta^2_n(i,a)<\infty,$ and $0\leq \beta_n(i,a) \leq 1$. 
\end{theorem}
\begin{proof}
The correspondence to Lemma \ref{fl} follows from associating $X$ with the set of state-action pairs $(i,a)$, $\beta_n(x)$ with $\beta_n(i,a)$, and $\Phi_n(i,a)$ with $Q_n(i,a)-Q^*_{w^*}(i,a)$, where $Q^*_{w^*}$ is the fixed point of $U_{w^*}$. Let the filtration for this process be defined by,  $\mathcal{G}_n=\{Q_0,i_j,a_j,\beta_j,w_j, \forall j\leq n, n\geq 0\}$. Note that  $\mathcal{G}_n \subseteq \mathcal{G}_{n+1}, \forall n \geq 0.$ 
The iteration of model-free SORQL is rewritten as follows
\begin{equation}\label{key23}
\begin{aligned}
\Phi_{n+1}(i_n,a_n) &= (1 - \beta_n(i_n,a_n)) \Phi_{n}(i_n,a_n) \\
&\quad + \beta_n(i_n,a_n) \left[ w_n \left( r^{i_{n+1}}_{i_n a_n} + \gamma \mathcal{M} Q_n(i_{n+1}) \right) \right. \\
&\quad \quad \left. + (1 - w_n) \left( \mathcal{M} Q_n(i_n) \right) - Q^*_{w^*}(i_n,a_n) \right].
\end{aligned}
\end{equation}
Without loss of generality, let $J_n(i,a) = w_n( r^j_{ia}+\gamma\, \mathcal{M}Q_n(j))+(1-w_n) (\mathcal{M}Q_n(i) )-Q^*_{w^*}(i,a).$
Also $\beta_0$, $\Phi_0$ is $\mathcal{G}_0$ measurable and $\beta_n$, $\Phi_n$, $J_{n-1}$ are $\mathcal{G}_{n}$ measurable for $n\geq1$. 
Define $J^{Q}_n= w^*( r^j_{ia}+\gamma\mathcal{M}Q_n(j))+(1-w^*) (\mathcal{M}Q_n(i) )$.
\noindent		
We now consider
\begin{align}
&\left| \E[J_n(i,a) \mid \mathcal{G}_n] \right| \nonumber\\
&= \left| \E \left[ J_n(i,a) + J^{Q}_n - J^{Q}_n + Q^*_{w^*}(i,a) - Q^*_{w^*}(i,a) \mid \mathcal{G}_n \right] \right| \nonumber\\
&\leq \left| U_{w^*} Q_n(i,a) - U_{w^*} Q^*_{w^*}(i,a) \right| \nonumber\\
&\quad \quad \quad \quad + \left| w_n - w^* \right| \left( R_{\max} + 2 \| Q_n \| \right) \nonumber\\
&\leq (1 - w^* + w^* \gamma) \| \Phi_n \| + \delta_n
\end{align}
where $\delta_n= |w_n-w^*| (R_{\max} + 2 \|Q_n\|)$. From Theorem \ref{thm1} and \textbf{(A1)}, $\delta_n$ $\rightarrow$ $0$ as $n \rightarrow \infty$.
Therefore, condition $(3)$ of Lemma \ref{fl} holds. Now we verify the condition (4) of Lemma \ref{fl}, 
	\begin{align*}	
	&Var[J_n(i,a)|\mathcal{G}_n]\\&=\E\left[\left(J_n(i,a)-\E[J_n(i,a)|\mathcal{G}_n]\right)^2 |\mathcal{G}_n\right]\\
	&\leq \E\left[\left(w_n(r^j_{ia}+\gamma \mathcal{M}Q_n(j))+(1-w_n) ( \mathcal{M}Q_n(i))\right)^2|\mathcal{G}_n\right]\\
	& \leq 3 \left(\kappa R_{\max}^2+ \kappa\gamma^2 \|Q_n\|^2+ \kappa_1  \|Q_n\|^2\right)\\
	& \leq  3\left(\kappa R_{\max}^2+ 2(\kappa\gamma^2 +\kappa_1 )  (\|\Phi_n\|^2 + \|Q^*_{w^*}\|^2)\right)\\
	& \leq  3\left(\kappa R_{\max}^2+ 2(\kappa\gamma^2 +\kappa_1) \|\Phi_n\|^2 +2(\kappa\gamma^2 +\kappa_1)  \|Q^*_{w^*}\|^2\right)\\
	& \leq  \left(3\kappa R_{\max}^2+6(\kappa\gamma^2 +\kappa_1)  \|Q^*_{w^*}\|^2+6(\kappa\gamma^2 +\kappa_1) \|\Phi_n\|^2 \right)\\
	& \leq C(1 + \|\Phi_n\|^2 ) \leq C(1 + \|\Phi_n\| )^2,
	\end{align*}
	where $C = \max\{3\kappa R_{\max}^2+6(\kappa\gamma^2 +\kappa_1)  \|Q^*_{w^*}\|^2,6(\kappa\gamma^2 +\kappa_1)\}$,  $w_n^2<\kappa$ and $(1-w_n)^2 <\kappa_1$. 

Consequently, all the conditions of Lemma \ref{fl} holds and hence $\Phi_{n} \rightarrow 0$ w.p.1, therefore, $Q_n \rightarrow Q^*_{w^{*}}$ w.p.1.
\end{proof}
Before going to discuss the convergence of the double SOR Q-learning algorithm mentioned in Remark \ref{remark2}.  A characterization result for the dynamics of $V^{BA}_n(i, a):=\widetilde{Q}^B_n(i,a)-\widetilde{Q}^A_n(i,a)$ will be presented. 
\begin{lemma}\label{lemma2}
	Consider the updates $\widetilde{Q}^A_n$ and $\widetilde{Q}^B_n$ as in Remark \ref{remark2}. Then
	$\E[V^{BA}_n(i,a)|\mathcal{G}_n]$ converges to zero w.p.1 as $n \rightarrow \infty$. 
\end{lemma}
\begin{proof}
	Suppose $V^{BA}_{n+1}(i,a):=\widetilde{Q}^B_{n+1}(i,a)-\widetilde{Q}^A_{n+1}(i,a)$. Then Remark \ref{remark2} indicates that at each iteration, either $\widetilde{Q}^A$ or $\widetilde{Q}^B$ is updated with equal probability. If $\widetilde{Q}^A$ is getting updated at $(n+1)^{th}$ iteration, then we have 
	\begin{align}
	&V^{BA}_{n+1}(i,a)\nonumber\\
	&\numeq{i}(1-\beta_n)V^{BA}_{n}(i,a)+\beta_n \big(\widetilde{Q}^B_{n}(i,a)-w(r^j_{ia}+\gamma \widetilde{Q}^B_{n}(j,b^*))\nonumber\\&\hspace*{1cm}-(1-w)\widetilde{Q}^B_{n}(i,c^*))\big).
	\end{align}
	Similarly if $\widetilde{Q}^B$ is getting updated at $(n+1)^{th}$ iteration, then  
	\begin{align}
	V^{BA}_{n+1}(i,a)
	&\numeq{ii}(1-\beta_n)V^{BA}_{n}(i,a)+\beta_n \big(w(r^j_{ia}+\gamma \widetilde{Q}^A_{n}(j,d^*))\nonumber\\&\hspace*{1cm}+(1-w)\widetilde{Q}^A_{n}(i,e^*))-\widetilde{Q}^A_{n}(i,a)\big).
	\end{align}
	Combining  $(i)$ and $(ii)$, we have
	\begin{align}
	V^{BA}_{n+1}(i,a)&=(1-\beta_n)V^{BA}_{n}(i,a)+\beta_n J^A_n(i,a)\\
	V^{BA}_{n+1}(i,a)&=(1-\beta_n)V^{BA}_{n}(i,a)+\beta_n J^B_n(i,a)
	\end{align}
	where
	\begin{equation}
	J^A_n(i,a)=\widetilde{Q}^B_{n}(i,a)-w(r^j_{ia}+\gamma \widetilde{Q}^B_{n}(j,b^*))-(1-w)\widetilde{Q}^B_{n}(i,c^*)
	\end{equation}
	\begin{equation}
	J^B_n(i,a)=w(r^j_{ia}+\gamma \widetilde{Q}^A_{n}(j,d^*))+(1-w)\widetilde{Q}^A_{n}(i,e^*))-\widetilde{Q}^A_{n}(i,a).
	\end{equation}
	Therefore
	\begin{align}
	&\E[V^{BA}_{n+1}(i,a)|\mathcal{G}_n]	\nonumber
	\\&=(1-\dfrac{\beta_n}{2})\E[V^{BA}_{n}(i,a)|\mathcal{G}_n]+\dfrac{\beta_n}{2}\E[w\gamma (\widetilde{Q}^A_{n}(j,d^*)\nonumber\\&\hspace*{0.1cm}-\widetilde{Q}^B_{n}(j,b^*))  +(1-w)(\widetilde{Q}^A_{n}(i,e^*)-\widetilde{Q}^B_{n}(i,c^*))|\mathcal{G}_n].
	\end{align}
	Let $\dfrac{\beta_n}{2}=\zeta_{n}$. Then, the above equation is reduced to
	\begin{align}\label{4}
	\E[V^{BA}_{n+1}(i,a)|\mathcal{G}_n]&= (1-\zeta_{n})\E[V^{BA}_{n}(i,a)|\mathcal{G}_n]+\zeta_{n}J^{BA}_n(i,a)
	\end{align}	
	where
	\begin{align}
	J^{BA}_n(i,a)&=\E[w\gamma (\widetilde{Q}^A_{n}(j,d^*)-\widetilde{Q}^B_{n}(j,b^*)) \nonumber\\&\quad \quad  +(1-w)(\widetilde{Q}^A_{n}(i,e^*)-\widetilde{Q}^B_{n}(i,c^*))|\mathcal{G}_n].
	\end{align}	
	Now for the above recursive Equation \eqref{4}, we verify the following conditions :
	(a) $\|\E[J^{BA}_n|\mathcal{G}_n]\|\leq (1-w+w\gamma)\|V^{BA}_n\| $ and (b) $Var[J^{BA}_n|\mathcal{G}_n]\leq C (1+\|V^{BA}_n\|)^2 $. Note that once we verify (a) and (b), then from Lemma \ref{fl}, we can conclude that the iterative scheme obtained in Equation \eqref{4} converges to zero w.p.1. Using the law of total expectation, the condition (b) is trivially satisfied. We now consider
	\begin{align}
	&|\E[J^{BA}_n(i,a)|\mathcal{G}_n]|\nonumber\\&=|\E[w\gamma (\widetilde{Q}^A_{n}(j,d^*)-\widetilde{Q}^B_{n}(j,b^*)) \nonumber\\&\hspace*{1cm}+(1-w)(\widetilde{Q}^A_{n}(i,e^*)-\widetilde{Q}^B_{n}(i,c^*))|\mathcal{G}_n]|
	\\& \numeqs{i} w\gamma \sum_{j=1, j\neq i}^{|S|}p^j_{ia}|\widetilde{Q}^A_{n}(j,d^*)-\widetilde{Q}^B_{n}(j,b^*)|\nonumber\\& \hspace{1cm}+ (w\gamma p^i_{ia}+(1-w))|(\widetilde{Q}^A_{n}(i,e^*)-\widetilde{Q}^B_{n}(i,c^*))|.
	\end{align}
	Since $0<w\leq w^*$, this implies $(1-w+w\gamma p^i_{ia}) \geq 0$. Therefore $(i)$ holds in the above equation. Clearly, at each step of the iteration either $\widetilde{Q}^A_{n}(j,d^*)\geq \widetilde{Q}^B_{n}(j,b^*)$ or $\widetilde{Q}^B_{n}(j,b^*)> \widetilde{Q}^A_{n}(j,d^*)$. Suppose $\widetilde{Q}^A_{n}(j,d^*)\geq \widetilde{Q}^B_{n}(j,b^*)$,  then
	\begin{align}
	&|\widetilde{Q}^A_{n}(j,d^*)-\widetilde{Q}^B_{n}(j,b^*)|\nonumber\\&=\widetilde{Q}^A_{n}(j,d^*)-\widetilde{Q}^B_{n}(j,b^*)\leq \widetilde{Q}^A_{n}(j,b^*)-\widetilde{Q}^B_{n}(j,b^*)
	\leq \|V^{BA}_n\|.
	\end{align}
	Similarly, when $\widetilde{Q}^B_{n}(j,b^*)> \widetilde{Q}^A_{n}(j,d^*)$. We have
	\begin{align}
	&|\widetilde{Q}^A_{n}(j,d^*)-\widetilde{Q}^B_{n}(j,b^*)|&\leq \widetilde{Q}^B_{n}(j,d^*)-\widetilde{Q}^A_{n}(j,d^*)
	\leq \|V^{BA}_n\|.
	\end{align}
	Therefore, by combining the above two inequalities, we obtain
	\begin{equation}
	|\widetilde{Q}^A_{n}(j,d^*)-\widetilde{Q}^B_{n}(j,b^*)|\leq \|V^{BA}_n\|.
	\end{equation}
	Similarly, one can show
	\begin{equation}
	|\widetilde{Q}^A_{n}(i,e^*)-\widetilde{Q}^B_{n}(i,c^*)|\leq \|V^{BA}_n\|.
	\end{equation}
	Hence, from Lemma \ref{fl}, $\E[V^{BA}_n|\mathcal{G}_n]$ $\rightarrow 0$ $w.p.1.$
\end{proof}
 Now we will show that the iterative scheme corresponding to the double SOR Q-learning converges under the following assumption: \textbf{(\textbf{A2}):} $\| \widetilde{Q}^A_n\| \leq B < \infty$ and $\| \widetilde{Q}^B_n\| \leq B < \infty$, $\forall n \geq 0$.

\begin{theorem}\label{thm4}
	Suppose \textbf{(A2)} holds and MDP be defined as in Section II. Let every state-action pair be sampled indefinitely. Then the double SOR Q-learning mentioned in Remark \ref{remark2} converges w.p.1 to the fixed point of $U_{w}$ as long as,  $\sum_{n}\beta_n(i,a)=\infty,\; \sum_{n}\beta^2_n(i,a)<\infty,$ and $0\leq \beta_n(i,a) \leq 1$.
\end{theorem}
\begin{proof}
	Due to the symmetry of the iterates, it is enough to show that $\widetilde{Q}^A$ $\rightarrow$ $Q^*_{w}$, which is the fixed point of $U_w$. Define $X = S \times A$ and $\Phi_n (i,a) = \widetilde{Q}^{A}_n(i,a)-Q^*_{w}(i,a)$. For the filtration $\mathcal{G}_n=\{\widetilde{Q}^A_0,\widetilde{Q}^B_0,i_j,a_j,\beta_j, \forall j\leq n, n\geq 0\}$. We have
	\begin{align}
	&\Phi_{n+1} (i_n,a_n)\nonumber\\& = (1-\beta_n(i_n,a_n))\Phi_n (i_n,a_n)+\beta_n(i_n,a_n)J_n(i_n,a_n)
	\end{align}
	where $J_n(i_n,a_n) = w (r^{i_{n+1}}_{i_na_n}+\gamma \widetilde{Q}^B_n(i_{n+1},b^*))+ (1-w)( \widetilde{Q}^B_n(i_n,c^*) )-Q^*_{w}(i_n,a_n)$. Without loss of generality, let
\begin{align}
&J_n(i,a) 
	\nonumber\\&= w (r^j_{ia}+\gamma \widetilde{Q}^B_n(j,b^*) )+ (1-w)\widetilde{Q}^B_n(i,c^*) -Q^*_{w}(i,a)\nonumber\\
 &= w (r^j_{ia}+\gamma \widetilde{Q}^A_n(j,b^*))+ (1-w) \widetilde{Q}^A_n(i,c^*)-Q^*_{w}(i,a)\nonumber\\
	&\quad +w\gamma (\widetilde{Q}^B_n(j,b^*)-\widetilde{Q}^A_n(j,b^*))\nonumber\\&\hspace*{1cm}+(1-w) (\widetilde{Q}^B_n(i,c^*)-\widetilde{Q}^A_n(i,c^*)).
	\end{align} 
	Note that $\beta_0$, $\Phi_0$ is $\mathcal{G}_0$ measurable and $\beta_n$, $\Phi_n$, $J_{n-1}$ are $\mathcal{G}_{n}$ measurable for $n\geq1$. Now
	\begin{align}
	&\left|\E[J_n(i,a)|\mathcal{G}_n]\right|\nonumber\\
	&\leq \bigg| U_{w}\widetilde{Q}^A_n(i,a)-U_{w}Q^*_{w}(i,a)\bigg|\nonumber\\&\hspace*{1cm}+ w\gamma \left|\E\left[(\widetilde{Q}^B_n(j,b^*)-\widetilde{Q}^A_n(j,b^*))|\mathcal{G}_n\right]\right|\nonumber\\
	&\hspace{1.4cm}+\left|(1-w)\right|\left| \E[(\widetilde{Q}^B_n(i,c^*)-\widetilde{Q}^A_n(i,c^*)) |\mathcal{G}_n]\right|\nonumber\\
	& \leq (1-w+w\gamma) \, \|\Phi_{n}\| + \delta_n
	\end{align}
	where $\delta_n =w\gamma \left|\E\left[(\widetilde{Q}^B_n(j,b^*)-\widetilde{Q}^A_n(j,b^*))|\mathcal{G}_n\right]\right|+\left|(1-w)\right|\left| \E[(\widetilde{Q}^B_n(i,c^*)-\widetilde{Q}^A_n(i,c^*)) |\mathcal{G}_n]\right|$. Therefore, $\delta_n = w \gamma |\E[V^{BA}_n(j,b^*)]|+|(1-w)| |\E[V^{BA}_n(i,c^*)]|.$ From Lemma \ref{lemma2}, and \textbf{(A2)} it is evident that $\delta_n \rightarrow 0$ as $n \rightarrow \infty$ $w.p.1.$ Once again, condition (4) of Lemma \ref{fl} holds similar to the fourth condition of Theorem \ref{thm3} and hence omitted. Now, all the conditions of Lemma \ref{fl} holds and hence $\Phi_{n} \rightarrow 0$ w.p.1, therefore, $\widetilde{Q}^A_n \rightarrow Q^*_{w}$ w.p.1.
\end{proof}
Before discussing the convergence of the model-free double SOR Q-learning algorithm, a characterization result for the dynamics of $W^{BA}_n(i, a):=Q^B_n(i,a)-Q^A_n(i,a)$ similar to Lemma \ref{lemma2} will be mentioned in the following lemma.
\begin{lemma}\label{lemma3}
	Consider the update rule $Q^A_n$ and $Q^B_n$ of model-free double SOR Q-learning as in Algorithm 1. Then
	$\E[W^{BA}_n(i,a)|\mathcal{G}_n]$ converges to zero w.p.1 as $n \rightarrow \infty$. 	
\end{lemma}
\begin{proof}
	The proof proceeds on similar lines of Lemma \ref{lemma2} and hence omitted.
\end{proof}

At last, the convergence of model-free double SOR Q-learning is outlined under the following assumption: \textbf{(\textbf{A3}):} $\| Q^A_n\| \leq B < \infty$ and $\| Q^B_n\| \leq B < \infty$, $\forall n \geq 0$.
\begin{theorem}
	Suppose \textbf{(A3)} holds and MDP be defined as in Section II. Let every state-action pair be sampled indefinitely. Then the model-free double SOR Q-learning with update rule as in Algorithm 1, converges w.p.1 to the fixed point of $U_{w^*}$ as long as $\sum_{n}\beta_n(i,a)=\infty,\; \sum_{n}\beta^2_n(i,a)<\infty,$ and $0\leq \beta_n(i,a) \leq 1$.
\end{theorem}
\begin{proof}
	The proof proceeds on similar lines as that of Theorem \ref{thm3} and Theorem \ref{thm4} and hence omitted.
\end{proof}

Now, we conclude this section with a theoretical and experimental analysis of the overestimation bias of the SOR Q-learning algorithm in comparison to the proposed algorithm.

\subsection{Theoretical Analysis of Over-estimation Bias in SORQL}
In Q-learning at every $n^{th}$ iteration the problem is to evaluate $\max_a \E [Q_n(s_{n+1},a)]$, i.e., for set of random variables $\{X_i\}^{d}_{i=1}$, we wish to compute $\max_i \E [X_i]$ \cite{hasselt2010double}. Let $\{\mu_i\}^{d}_{i=1}$ be the unbiased estimators of $\{X_i\}^{d}_{i=1}$. The estimator is unbiased in the sense that $\E[X_i]=\E[\mu_i], \forall i$. Therefore, Bias $= \E[\max_i \mu_i] - \max_i \E[\mu_i]$, but from Jensen's inequality, we have $\E[\max_i \mu_i] \geq \max_i \E[\mu_i]$, which makes Q-learning to over-estimate. Similarly, the estimator of the SORQL algorithm, which is $w\max_i \mu_i + (1-w) \max_j \zeta_j$, where $\zeta_j$ are unbiased estimators with respect to the set of random variables $\{Y_j\}^{d}_{j=1}$ corresponding to the actions in the previous state. Then the bias is given by \begin{equation}\label{bias}
\text{Bias} = \E[w\max_i \mu_i + (1-w) \max_j \zeta_j]-\max_i \E[\mu_i] .
\end{equation}
For $w>1$ in SORQL, we have $p^i_{ia}>0, \forall (i,a)\in S\times A$. Hence, with probability $p^i_{ia}$, $X_i=Y_j$. Consequently, the above equation reduces to
\begin{align}
\text{Bias} &= \E[w \max_i \mu_i + (1-w) \max_i \mu_i] - \max_i \E[\mu_i] \\
&= \E[\max_i \mu_i] - \max_i \E[\mu_i] \geq 0.
\end{align}
Therefore, with probability $p^i_{ia}$, the SORQL algorithm overestimates the true value. Similarly, we can see that the MF-SORQL will also suffer from over-estimation bias. To this end, the techniques developed in double Q-learning prove useful. More precisely, in double Q-learning, for a set of random variables $\{X_i\}^{d}_{i=1}$, to evaluate $\max_i \E [X_i]$, two sets of unbiased estimators, namely $\{\mu^A_i\}^{d}_{i=1}$, and $\{\mu^B_i\}^{d}_{i=1}$ are used. Further, it is shown that for $a^*$ such that $\mu^A_{a^*}= \max_i \mu^A_i$, $ \E[\mu^B_{a^*}] \leq \max_i \E [X_i]$ (Lemma 1, \cite{hasselt2010double}). Therefore, Bias $= \E[ \mu^B_{a^*}] - \max_i \E[\mu^B_i]\leq 0$, consequently, double Q-learning underestimates the true value. Note that underestimation is not a complete solution to the over-estimation problem of Q-learning. However, this modification to Q-learning has proven to be more useful in many situations \cite{van2016deep}. Now, for the proposed DSORQL, let $\{\mu^A_i\}^{d}_{i=1}$, $\{\mu^B_i\}^{d}_{i=1}$, and $\{\zeta^A_j\}^{d}_{j=1}$, $\{\zeta^B_j\}^{d}_{j=1}$ be the unbiased estimators of $\{X_i\}^{d}_{i=1}$, and  $\{Y_j\}^{d}_{j=1}$, respectively. Therefore, the bias is given by
\begin{equation}
\text{Bias} = \E [w \mu^B_{a^*} + (1-w) \zeta^B_{b^*}]-\max_i \E[\mu^B_i]
\end{equation}
  where $a^*$ and $b^*$ is such that $\mu^A_{a^*}= \max_i \mu^A_i$ and $\zeta^A_{b^*}= \max_j \zeta^A_j$ respectively. 
Similar to the earlier argument, for $w>1$ in DSORQL, we have $p^i_{ia}>0, \forall (i,a)\in S\times A$. Hence, with probability $p^i_{ia}$, $X_i=Y_j$. Consequently, the above equation reduces to
\begin{align}
\text{Bias} &= \E[w \mu^B_{a^*} + (1-w)\mu^B_{a^*}]- \max_i \E[\mu^B_i] \\
&=\E[\mu^B_{a^*}]- \max_i\E[\mu^B_i] \leq 0 \quad (\text{From Lemma 1, \cite{hasselt2010double}}). 
\end{align}

Hence, DSORQL underestimates the true action value. Similarly, one can observe that model-free DSORQL also underestimates the true action value.  
\subsection{Experimental Analysis of Over-estimation Bias in SORQL using Roulette \cite{hasselt2010double}}
In roulette, a player can choose from $170$ different betting actions, such as betting on a specific number or on the color black or red. Each bet has a payout designed so that the expected return is approximately \$0.947 per dollar wagered, resulting in an expected loss of \$-0.053 per play. This is modeled as an MDP with one state and $171$ actions, where one action is not placing a bet, ending the episode. If the player chooses not to bet, the payout is \$0. Assuming the player bets \$1 each time, the optimal strategy is to refrain from gambling, resulting in an expected profit of \$0. In this example, the choice of successive relaxation factor $w$ is known as we have only one state. More specifically, $0 < w \leq w^* = \frac{1}{1-\gamma\min_{i,a}p^i_{ia}}=\frac{1}{1-\gamma}$. Therefore, both MF-DSORQL and MF-SORQL are the same as that of DSORQL and SORQL, respectively.
\vspace{-0.3cm}
 \begin{figure}[ht]
	\centering
	\includegraphics[width=0.4\textwidth]{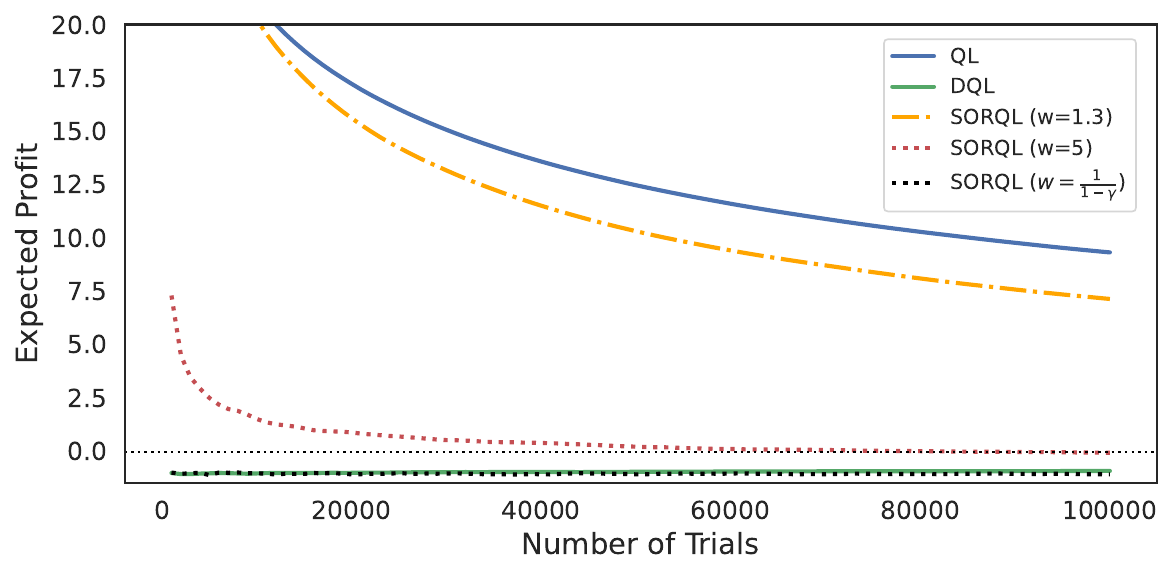}
	\vspace{-0.4cm}
	\caption{SORQL for different SOR parameters.}
	\label{fig:image1}
\end{figure}
\vspace{-0.5cm}
\begin{figure}[ht]
	\centering
	\includegraphics[width=0.4\textwidth]{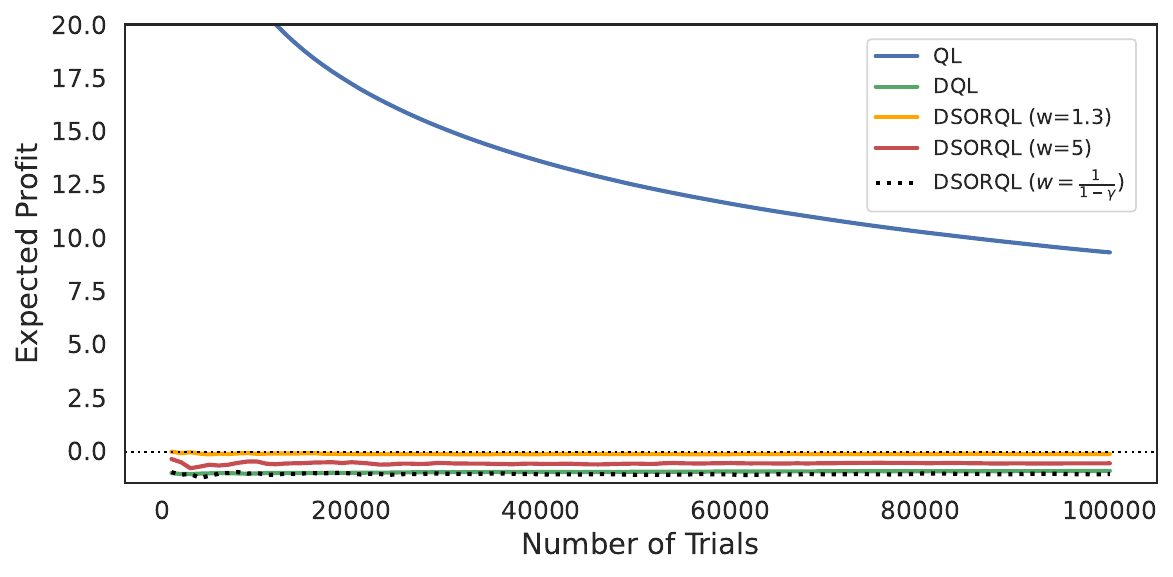}
	\vspace{-0.4cm}
	\caption{DSORQL for different SOR parameters.}
	\label{fig:image2}
\end{figure}\\

The experiment is conducted with $\gamma=0.95$, and the polynomial step-size mentioned in \cite{hasselt2010double}.  Fig. \ref{fig:image1} and Fig. \ref{fig:image2} illustrate the performance of the SORQL and DSORQL algorithms, respectively, for different choices of SOR parameters, and compare them with QL \cite{watkins1992q} and DQL \cite{hasselt2010double}. Each experiment involves $100000$ trials, and the graphs are obtained by averaging the results over $10$ independent runs. Following \cite{hasselt2010double} all algorithms are updated synchronously and trained using epsilon-greedy policy.
\noindent
From Fig. \ref{fig:image2}, it is evident that the performance of DSORQL remains consistent with that of DQL for various values of $w$. In contrast, Fig. \ref{fig:image1}  clearly demonstrates the overestimation bias in SORQL algorithm.

\section{Numerical Experiments} 
In this section, the tabular and deep RL versions of the proposed algorithms are compared on benchmark examples. {The detailed experimental setup is in the supplementary material. }The code for all the numerical experiments in this manuscript is publicly available and can be accessed in the author's GitHub repository \cite{shreyas}.

\subsection{Tabular Version}
This subsection provides numerical comparisons between Q-learning (QL) \cite{watkins1992q}, double Q-learning (DQL) \cite{hasselt2010double}, triple-average policy gradient (TAPG) \cite{wu2020reducing},  model-free SOR Q-learning (MF-SORQL), and model-free double SOR Q-learning (MF-DSORQL). More specifically, we compare the above-mentioned algorithms on benchmark examples  such as  CartPole \cite{brockman2016openai}, LunarLander \cite{brockman2016openai}, and multi-armed bandit example discussed in \cite{8695133}, and \cite{tan2024q}. The choice of SOR parameter for SORQL and DSORQL are kept same, and the parameter $\beta$ for TAPG is set to 0.95 as discussed in \cite{wu2020reducing}.

\subsubsection{CartPole Environment} 
In the CartPole-v0 environment, the goal is to balance a pole on a moving cart by applying forces in the left and right direction on the cart. The agent receives a reward of $+1$ for each time step the pole remains upright, and the episode terminates if the pole angle is not in the range of $\pm 12^{\circ}$ or the cart range is not within $(-2.4,2.4)$. In this experiment,  the discount factor is 0.999, and the step size is $\frac{40}{n+100}$. We use an epsilon-greedy policy to train the algorithms. The continuous CartPole environment is discretized to $72$ states. The standard threshold for a successful CartPole-v0 agent is to achieve an average reward of $195$ over $50$ episodes. Table \ref{tab:my-table}, provides the minimum number of episodes required to achieve an average reward of $195$ over $50$ episodes. From Table \ref{tab:my-table}, it is evident that the proposed DSORQL algorithm, with the SOR parameter $w=1.1$, achieves the target much quicker than the rest of the algorithms.

\subsubsection{ LunarLander Environment}

The LunarLander-v2 environment involves controlling a spacecraft to land safely on a lunar surface. The agent is rewarded for maintaining the lander within the landing pad, controlling its orientation, and landing with minimal velocity. This environment serves as a test for algorithms to demonstrate the ability to make decisions at varying complexity levels. The threshold for LunarLander-v2 is to achieve a reward of $200$. The following table, i.e., Table \ref{tab:my-table}, provides the minimum number of episodes required to achieve the threshold averaged over five independent experiments. The step-size for all the algorithms are same with the discount factor is set to $0.99$. Once again, from Table \ref{tab:my-table} it is clear that the proposed algorithm with the SOR parameter $w=1.3$ learns the optimal policy much quicker than the rest of the algorithms. 
\begin{table}[ht]
	\centering
	\resizebox{0.48\textwidth}{!}{%
		\begin{tabular}{lccccc}
			\toprule
			\textbf{Algorithm} & \textbf{QL \cite{watkins1992q}} & \textbf{SORQL \cite{kamanchi2019successive}} & \textbf{TAPG \cite{wu2020reducing}} & \textbf{DQL \cite{hasselt2010double}} & \textbf{DSORQL} \\
			\midrule
			\textbf{CartPole}      & 498  & 648  & 1000  & 1000  & \textbf{298} \\
			\textbf{LunarLander}   & 1065 & 1128 & 1073  & 1085  & \textbf{905} \\
			\bottomrule
		\end{tabular}
	}
	\vspace{0.1cm} 
	\caption{Number of episodes required to complete.}
	\label{tab:my-table}
\end{table}

\subsubsection{Multi-Armed Bandit Example \cite{8695133, tan2024q}}

This example considers a scenario with a single state and multiple action options. Specifically, there are thirty-nine available actions, i.e., $|A| = 39$, and the thirty-ninth action leads to the end of the episode with no reward (See, Fig. \ref{key3}). The other thirty-eight actions provide a reward sampled from a Gaussian distribution with a mean of $-0.0526$ and a variance of one. In this example, the discount factor, $\gamma$, is set to $0.99$, and the step size for all algorithms is defined as $100/(n+100)$. All state-action values are initialized to zero.  Training is conducted over $50000$ episodes, averaging the results across $10$ independent runs. Similar to \cite{8695133}, all algorithms are updated asynchronously and trained using an exploratory policy.

 The following table, i.e., Table \ref{tab:approx_q_values}, presents the performance comparison of the proposed DSORQL algorithm with QL, DQL, SORQL, and TAPG. Note that the optimal Q-value in this example is zero.
 \begin{table}[ht]
 	\centering
 	\resizebox{0.48\textwidth}{!}{%
 		\begin{tabular}{@{}lcccccc@{}}
 			\toprule
 			\textbf{Algorithm} & \textbf{QL \cite{watkins1992q}} & \textbf{SORQL \cite{kamanchi2019successive}} & \textbf{TAPG \cite{wu2020reducing}} & \textbf{DQL \cite{hasselt2010double}} & \textbf{DSORQL} \\ \midrule
 			\textbf{$\max_a Q(a)$} & 8.140 & 1.443 & -5.27 & -5.61 & \textbf{0.078}  \\ 
 			\bottomrule
 		\end{tabular}%
 	}
 	\vspace{0.1cm}
 	\caption{Approximate $Q^*(a)$ after 50,000 episodes.}
 	\label{tab:approx_q_values}
 \end{table}
 
From Table \ref{tab:approx_q_values}, it is easy to see that DSORQL algorithm converges quicker to the optimal value in comparison to  the rest of the algorithms. The choice of SOR parameter in this example for both SORQL and DSORQL is $w=1/(1-\gamma)$. Also, during our numerical experiments, we observed that this behavior of DSORQL remained consistent even when the cardinality of the action space was increased or decreased, as well as when the action distribution is modified to $N(-0.1,1)$, similar to \cite{tan2024q}. This highlights the robustness of the proposed algorithm.

\begin{figure}[ht]
			\includegraphics[scale=0.17]{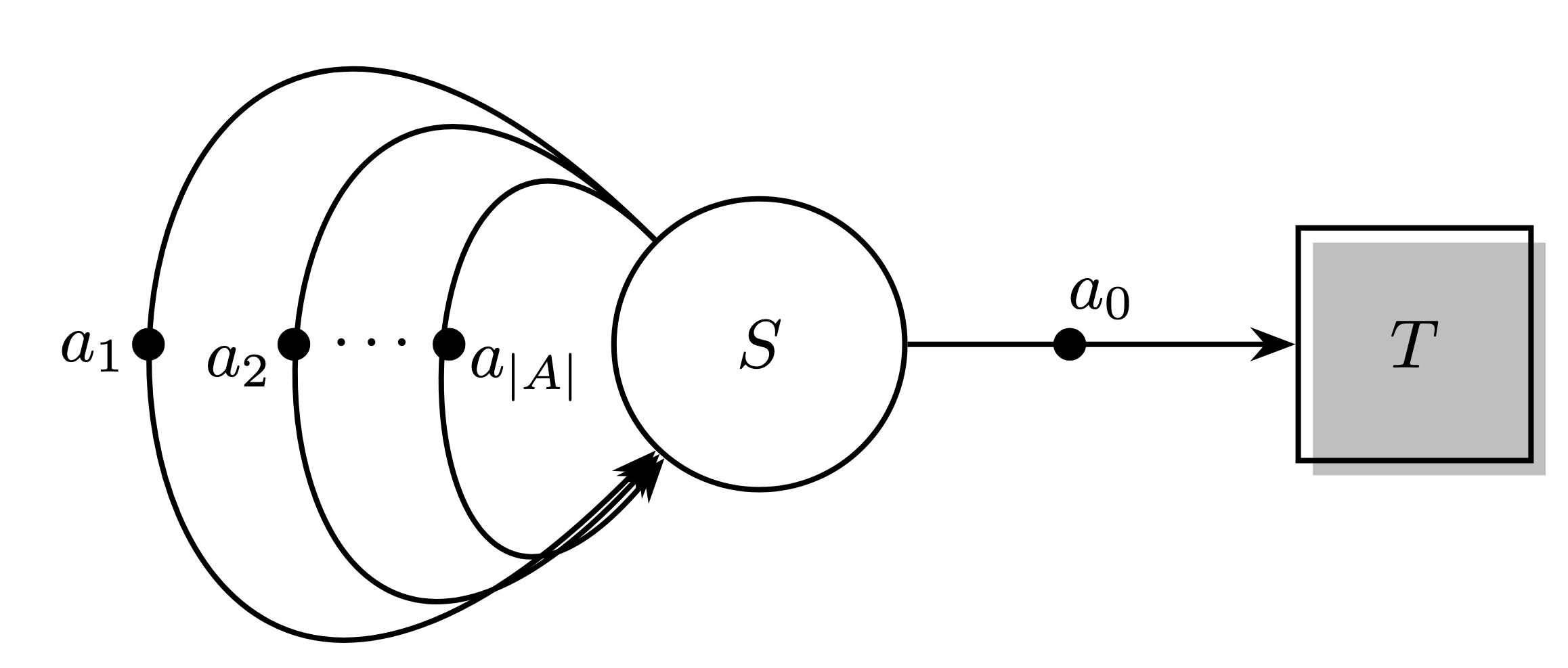}
				\centering
				\caption{Multi-Armed Bandit Example \cite{8695133}.}
				\label{key3}
\end{figure}
\subsection{Deep RL Version}
This subsection compares the proposed double  SOR Q-network (DSORDQN) with the deep RL version of SOR Q-learning (SORDQN) \cite{9206598}, the deep Q-network (DQN) \cite{mnih2013playing}, deep double Q-network (DDQN) \cite{van2016deep}, and Rainbow \cite{hessel2018rainbow}. The performance is evaluated using the maximization bias example discussed in \cite{weng2020mean},  and six Atari games.  We have used MushroomRL library \cite{d2021mushroomrl} to test the performance on the Atari games. The choice of SOR parameter $w$ in SORDQN and hyper-parameters in Rainbow are as discussed in \cite{9206598} and \cite{hessel2018rainbow}, respectively.  
\subsubsection{Maximization Bias Example \cite{weng2020mean}} This example is similar to the classic scenario used to illustrate bias behavior in tabular QL and DQL, which is discussed in Chapter 6 of \cite{MR3889951}. However, the current implementation is adapted from \cite{weng2020mean}. Specifically, the maximization bias example in this manuscript involves $\{0,1,2,..., M\}$ states with $M=10^9$ and two terminal state. Therefore, applying tabular methods is infeasible. Consequently, we apply the proposed DSORDQN algorithm and compare its performance with DQN, SORDQN, and DDQN. The agent always begins in state zero. The discount factor is $\gamma=0.999$. Each state allows two actions: moving left or right (See, Fig. \ref{fig:enter-label}). If the agent moves right from state zero, it receives zero reward, and the game ends. If the agent moves left, it transitions to one of the states in $\{1,2,...,M\}$ with equal probability and receives zero reward. In these new states, moving right returns the agent to state zero, while moving left ends the game, with rewards drawn from a Gaussian distribution $N(-0.1,1)$. 

Note that the choice of parameters in this example is the same as that of \cite{weng2020mean} and discussed in detail in the supplementary material. The policy used is epsilon-greedy with $\epsilon = 0.1$, with training conducted over $400$ episodes. We plot the probability of selecting the left action from state zero at the end of each episode, averaging over $1000$ runs. A higher probability of choosing left suggests that the algorithm is following a sub-optimal strategy, as consistently choosing right yields the highest mean reward. In Fig. \ref{fig:your_label}, one can observe that the policy obtained from the proposed algorithm is better than DQN, SORDQN, and DDQN. 
\begin{figure}[ht]
	\centering
	\begin{tikzpicture}[node distance=2cm, auto, thick, scale=0.7, transform shape]  
	\node[draw, rectangle, fill=gray!30, text=black, drop shadow] (start) at (0,0) {\textbf{T}};
	\node[draw, circle,  drop shadow, fill=white, label=above:{}] (1) at (2.4,2) {1};
	\node[draw, circle,  drop shadow, fill=white] (2) at (2.4,1) {2};
	\node[draw, circle, drop shadow, fill=white] (m) at (2.4,-1.5) {$M$};
	\node[draw, circle,  drop shadow, fill=white] (0) at (5,0) {0};
	\node[draw, rectangle, fill=gray!30, text=black, drop shadow] (end) at (8,0) {\textbf{T}};
	
	\draw[->, line width=1.5pt] (1) -- (start);
	\draw[->, line width=1.5pt] (2) -- (start);
	\draw[->, line width=1.5pt] (m) -- (start);
	\draw[<->, line width=1.5pt] (1) -- (0)  node[midway, right]{\textbf{Reward = 0}};
	\draw[<->, line width=1.5pt] (2) -- (0);
	\draw[<->, line width=1.5pt] (m) -- (0);
	\draw[->, line width=1.5pt] (0) -- (end) node[midway, below]{\textbf{Reward = 0}};
	
	\node at (2.4,-0.1) {$\vdots$};
	
	\filldraw (3.5,1.15) circle (2pt);
	\filldraw (3.5,0.55) circle (2pt);
	\filldraw (3.5,-0.85) circle (2pt);
	\filldraw (6.5,0) circle (2pt);
	
	\filldraw (1.2,1) circle (2pt);
	\filldraw (1.2,0.5) circle (2pt);
	\filldraw (1.2,-0.8) circle (2pt);
	
	\end{tikzpicture}
	
	\caption{Any action from the states $1$ to $M$ gives reward sampled from $N(-0.1,1)$, where $M = 10^9$.}
	\label{fig:enter-label}
\end{figure}

\subsubsection{Atari Games}

We evaluate the performance of DRL algorithms on six Atari 2600 games—Breakout, Qbert, Seaquest, Freeway, Asterix, and Space Invaders. The complete parameter set-up is available in the supplementary material. We report the mean reward over 10 epochs with standard deviation across five independent seeds, as summarized in Table \ref{tab:atari-comparison}. 

	The results in Table \ref{tab:atari-comparison} demonstrate that the proposed DSORDQN performs better compared to all baseline methods (DQN, SORDQN, DDQN) and Rainbow in four out of the five evaluated environments. Specifically, it achieves the highest average rewards in Seaquest (304.08 $\pm$ 5.64), Breakout (61.77 $\pm$ 2.75), Space Invaders (316.97 $\pm$ 15.01), Asterix (560.33 $\pm$ 15.03), and Freeway (25.04 $\pm$ 0.70). It also outperforms Rainbow by 28\% in Seaquest, 219\% in Breakout, 24\% in Space Invaders, 7\% in Asterix, and 67\% in Freeway. In addition, it demonstrates lower variance across multiple runs, indicating more reliable learning outcomes.

	While Rainbow achieves the highest reward in Qbert, it is a multi-step algorithm requiring more training samples, has several hyperparameters to tune, and our numerical experiments indicate that it takes longer to train due to its increased complexity. In contrast, DSORDQN provides a more efficient learning framework, yielding consistently better results across rest of the environments while maintaining computational feasibility.
	
	\begin{figure}[ht]
		\centering
		\includegraphics[width=8cm,height=4cm]{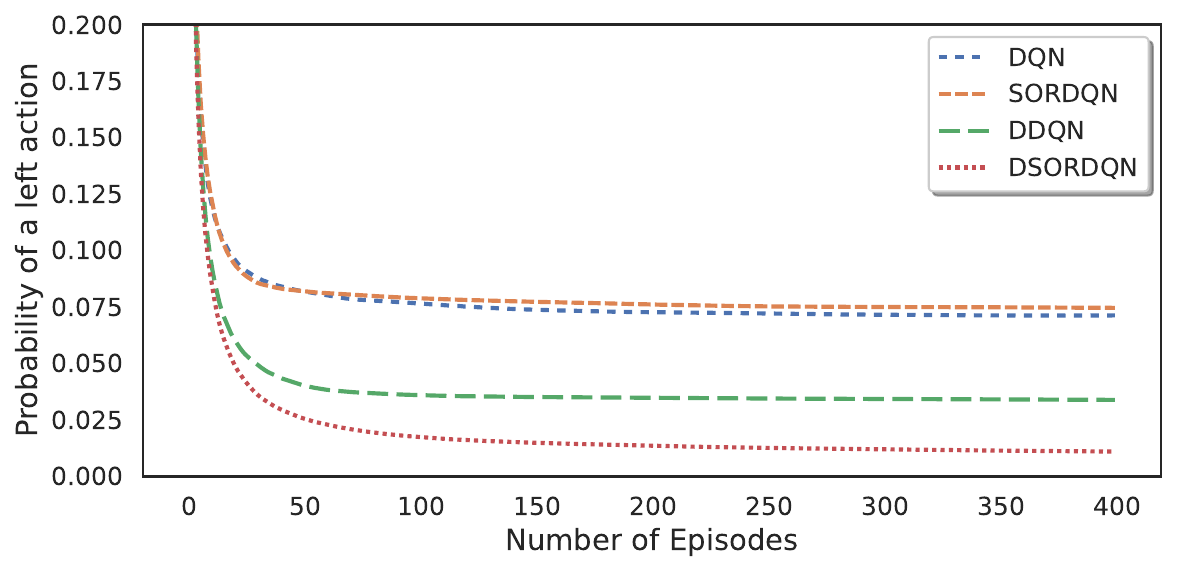}
		\caption{The performance of algorithms on maximization bias example.}
		\label{fig:your_label}
	\end{figure}
	
We conclude this manuscript with a discussion on the ablation studies.
\subsection{Ablation Studies}
In Fig.~\ref{fig:image1} and Fig.~\ref{fig:image2}, we present a comparison between the SORQL and DSORQL algorithms across different values of the SOR parameter ($w$) on the roulette example. Specifically, we test $w=1.3$, $w=5$, and $w=\frac{1}{1-\gamma}$. The plots demonstrate that, unlike SORQL (Fig.~\ref{fig:image1}), DSORQL consistently performs well across these varying choices of $w$ (Fig.~\ref{fig:image2}). This result suggests that DSORQL is more robust to changes in $w$ compared to SORQL.

In the CartPole environment, we found that $w=1.1$ outperforms other choices of $w$ (e.g., $w=1.2$, $w=1.3$, and $w=1/(1-\gamma)$), which shows the importance of selecting the right $w$ for optimal performance. A similar trend is observed in the LunarLander domain, where $w=1.3$ delivers the best performance. For the multi-armed bandit task, we observe that performance improves as $w$ approaches $1/(1-\gamma)$.

In the maximization bias example, similar to the multi-armed bandit task, we again see performance improvements as $w$ is varied from $1.1$ to $1/(1-\gamma)$. 
For Atari games, experiments suggest that $w=1.2$ and $w=1.3$ yield promising performance on the discussed Atari tasks. Note that $w=1$ corresponds to DQL in the tabular setting and DDQN in the function approximation setting. From theory, it can be observed that faster convergence is achievable for $w>1$ \cite{9206598}, and hence, we choose the SOR parameter accordingly to optimize convergence. In our numerical experiments, we obtain a suitable SOR parameter by testing various values of $w$. Alternatively, hyperparameter tuning methods, such as those described in \cite{feurer2019hyperparameter}, can be applied to automate and refine this process.

\section{Conclusion}
Motivated by the recent success of SOR methods in Q-learning and the issues concerning its over-estimation bias, this manuscript proposes a double successive over-relaxation Q-learning. Both theoretically and empirically, the proposed algorithm demonstrates a lower bias compared to SOR Q-learning. In the tabular setting, convergence is proven under suitable assumptions using SA techniques. A deep RL version of the method is also presented. Experiments on standard examples, showcase the algorithm's performance. It is worth mentioning that, similar to the boundedness of iterates observed in Q-learning \cite{gosavi2006boundedness}, investigating the boundedness of the proposed algorithms in the tabular setting presents an interesting theoretical question for future research. 

\begin{table*}[ht]
	\centering
	\resizebox{0.8\textwidth}{!}{%
		\begin{tabular}{lccccc}
			\toprule
			\textbf{Atari Games} & \textbf{DQN \cite{mnih2013playing}} & \textbf{SORDQN \cite{9206598}} & \textbf{Rainbow \cite{hessel2018rainbow}} & \textbf{DDQN \cite{van2016deep}} & \textbf{DSORDQN} \\ 
			\midrule
			\textbf{Seaquest}       & 143.13 (37.85) & 202.18 (19.77)  & 237.05 (12.20)           & 208.30 (60.02) & \textbf{304.08 (5.64)}  \\ 
			\textbf{Breakout}       & 55.34 (3.70)   & 4.27 (2.43)     & 19.36 (1.10)             & 54.02 (4.60)   & \textbf{61.77 (2.75)}   \\ 
			\textbf{Space Invaders} & 249.58 (11.62) & 251.71 (14.69)  & 254.72 (8.70)            & 247.99 (31.03) & \textbf{316.97 (15.01)} \\ 
			\textbf{Qbert}          & 552.78 (40.38) & 267.18 (39.23)  & \textbf{2957.80 (86.33)} & 668.86 (58.26) & 830.62 (31.71)          \\ 
			\textbf{Asterix}        & 437.59 (54.30) & 236.62 (27.33)  & 524.51 (24.21)           & 408.11 (36.09) & \textbf{560.33 (15.03)} \\ 
			\textbf{Freeway}        & 11.91 (3.88)   & 18.13 (1.62)    & 14.92 (1.79)             & 18.81 (1.06)   & \textbf{25.04 (0.70)}   \\ 
			\bottomrule
		\end{tabular}%
	}
	\vspace{0.1cm}
	\caption{Mean test scores (with standard deviation in parentheses) over ten epochs, averaged across five random seeds.}
	\label{tab:atari-comparison}
\end{table*}

\section*{Acknowledgement}
I would like to sincerely thank Dr. Neha Bhadala and the anonymous referees for carefully reading the manuscript and providing constructive feedback, which greatly enhanced its presentation and overall quality.

 \section*{Supplementary Material}
 
 This section provides details on the network architecture, training setup, and evaluation methodology for the deep reinforcement learning (DRL) experiments. Finally, the ablation studies for all numerical experiments are discussed.
 
 \subsection{Network Architecture, Training Setup, and Evaluation Methodology}
 
 Table I presents parameters for the maximization bias example, Table II presents the network architecture and training parameters for Atari, and Table III provides the parameters for the Rainbow, SORDQN, and DSORDQN algorithms. 
 
 \begin{table}[h]
 	\caption{Neural network and training parameters used for the maximization bias example. The parameters are similar to \cite{weng2020mean}.}
 	\label{tab:dql-parameters}
 	\centering
 	\begin{tabular}{lc}
 		\toprule
 		\textbf{Parameter} & \textbf{Value} \\
 		\midrule
 		Network architecture & Fully connected \\
 		Hidden layers & 2 \\
 		Layer sizes & 4, 8 \\
 		Activation & ReLU \\
 		Optimizer & SGD \\
 		Discount factor ($\gamma$) & 0.999 \\
 		Exploration & Epsilon-greedy \\
 		Epsilon ($\epsilon$) & 0.1 \\
 		States & $10^9 + 2$ \\
 		Actions & 2 (Left, Right) \\
 		Episodes & 400 \\
 		Iterations & 1000 \\
 		\bottomrule
 	\end{tabular}
 \end{table}

 \begin{table}[ht]
 	\caption{Network architecture and training parameters for Atari experiments.}
 	\label{tab:network-params}
 	\centering
 	\begin{tabular}{ll}
 		\toprule
 		\textbf{Component} & \textbf{Value} \\
 		\midrule
 		Env. & 6 Atari games via MushroomRL \\
 		Input & 84×84 grayscale (4-frame stack) \\
 		CNN & 3 conv layers + 2 FC layers \\
 		Conv layers & 32 (8×8/4), 64 (4×4/2), 64 (3×3/1) \\
 		FC layer & 512 units, action-value output \\
 		Loss & Huber (smooth L1) \\
 		Policy & Epsilon-greedy ($\epsilon$: 1.0 → 0.1) \\
 		Eval $\epsilon$ & 0.05 \\
 		Hist. length & 4 \\
 		Train freq. & 4 steps \\
 		Eval freq. & 250,000 steps \\
 		Target update & Every 10,000 steps \\
 		Replay init size & 50,000 \\
 		Max replay size & 500,000 \\
 		Test samples & 125,000 \\
 		Max train steps & 2,500,000 \\
 		Batch size & 32 \\
 		Discount factor ($\gamma$) & 0.99 \\
 		Optimizer & Adam ($\alpha$ = 0.00025) \\
 		Metrics & Mean reward (10 epochs), std over 5 seeds \\
 		\bottomrule
 	\end{tabular}
 \end{table}

 \begin{table}[ht]
 	\caption{Parameters for Rainbow \cite{hessel2018rainbow}, SORDQN \cite{9206598}, and DSORDQN on Atari.}
 	\label{tab:rainbow-parameters}
 	\centering
 	\begin{tabular}{lc}
 		\toprule
 		\textbf{Parameter} & \textbf{Value} \\
 		\midrule
 		Atoms ($n_{\text{atoms}}$) & 51 \\
 		Value range ($v_{\min}$, $v_{\max}$) & (-10, 10) \\
 		Multi-step return ($n$) & 3 \\
 		Prioritization exponent ($\alpha$) & 0.6 \\
 		Priority correction ($\beta$) & Linear (0.4 to 1) \\
 		Noisy net parameter ($\sigma$) & 0.5 \\
 		SOR param ($w$) & 1.3 \\
 		\bottomrule
 	\end{tabular}
 \end{table}
 
\bibliographystyle{IEEEtran}
\bibliography{ref.bib}

\end{document}